\documentclass{article}

\usepackage[preprint,nonatbib]{Preprint_MMMS}
\usepackage{amssymb}
\usepackage{amsthm}
\usepackage{mathtools}
\usepackage{amsfonts}
\usepackage{amsmath}

\usepackage[utf8]{inputenc} 
\usepackage[T1]{fontenc}    
\usepackage{hyperref}       
\usepackage{url}            
\usepackage{booktabs}       
\usepackage{amsfonts}       
\usepackage{nicefrac}       
\usepackage{microtype}      
\usepackage{xcolor}         

\newtheorem*{theorem*}{Theorem}
\newtheorem{theorem}{Theorem}[section]
\newtheorem{definition}[theorem]{Definition}

\newtheorem{lemma}[theorem]{Lemma}

\newtheorem{corollary}[theorem]{Corollary}

\title{Generalization for slowly mixing processes}

\author{%
	Andreas Maurer \\
	Istituto Italiano di Tecnologia, CSML, 16163 Genoa, Italy\\
	\texttt{am@andreas-maurer.eu}
}

\begin{document}

	\maketitle

\begin{abstract}
For samples generated by stationary and $\varphi $-mixing processes we give
generalization guarantees uniform over Lipschitz, smooth or unconstrained
loss classes. The result depends on an empirical estimator of the distance
of data drawn from the invariant distribution to the sample path. The mixing
time (the time needed to obtain approximate independence) enters these
results explicitely only in an additive way. For slowly mixing processes
this can be a considerable advantage over results with multiplicative
dependence on the mixing time. Because of the applicability to unconstrained
loss classes, where the bound depends only on local Lipschitz properties at
the sample points, it may be interesting also for iid processes, whenever
the data distribution is a sufficiently simple object.

\end{abstract}

\section{Introduction}

A key problem in learning theory is to give performance guarantees on new,
yet unseen data for hypotheses which are selected on the basis of their
performance on a finite number of observations. Standard methods which have
been developed for this purpose are the VC-Theory (\cite%
{vapnik1971chervonenkis}), the method of Rademacher averages (\cite%
{Koltchinskii00},\cite{Bartlett02}), the concept of algorithmic stability (%
\cite{Bousquet02}) and the Pac-Bayesian bounds (\cite{mcallester1999pac}). A
common feature of these methods is, that the observable performance of a
hypothesis is measured by the incurred losses averaged over the
observations, the empirical risk, a choice rooted in mathematical tradition
since Cardano stated the law of large numbers in the 16th century.

While these methods are quite effective when the observations are
independent, there are difficulties when they become dependent, for example
in the study of dynamical systems. A popular approach assumes the process to
be stationary and mixing, so that future observations are nearly independent
of a sufficiently distant past.

If the $X_{i}$ are observations becoming approximately independent after $%
\tau $ increments of time, then $\left( X_{1},X_{1+\tau },X_{1+2\tau
},...,X_{1+\left( n-1\right) \tau }\right) $ can be treated as a vector of $n
$ independent observations, to which the law of large numbers can be
applied, modulo a correction term dependent on $\tau $. An approach based on
this idea using nearly independent data blocks has been introduced by \cite%
{yu1994rates} and since been used in various forms by many authors (\cite%
{meir2000nonparametric}, \cite{mohri2008rademacher}, \cite{steinwart2009fast}%
, \cite{mohri2010stability}, \cite{agarwal2012generalization}, \cite%
{shalizi2013predictive}, \cite{wolfer2019minimax} and others) to port
established techniques and results from the independent to the dependent
setting.

A practical limitation of this approach is that, to obtain the same bound on
the estimation error as for independent data, the number of necessary
observations is multiplied with the mixing time $\tau $ or estimates
thereof. This is a major problem when the process mixes very slowly. This
paper proposes an alternative method to prove generalization guarantees,
which circumvents this problem and may be interesting also in the case of
iid processes.

The proposed bound is uniform on very large, even unconstrained loss
classes, but it requires \textit{easy} data. The data distribution, or
invariant distribution of the process, should be a small, essentially low
dimensional object, but it may be embedded in a high- or
infinite-dimensional space in an intricate, nonlinear way. This hidden
simplicity, related to the manifold hypothesis (\cite{ma2012manifold}, \cite%
{fefferman2016testing}, \cite{berenfeld2021density}), is not postulated as
an unverifiable assumption, but it is unveiled by the data-dependent nature
of the bound and directly observable. The underlying intuition is the
following. If most of the mass of the distribution is close to sample path,
then any function of the loss class, which doesn't change much in
neighborhoods of the observed points, should generalize well.

In its simplest form, when the loss class $\mathcal{F}$ consists of $L$%
-Lipschitz functions on a metric space $\left( \mathcal{X},d\right) $ with
sample path $\mathbf{X}=\left( X_{1},...,X_{n}\right) $, the proposed bound
says that with probability at least $1-\delta $ in $\mathbf{X}$ every $f\in 
\mathcal{F}$ satisfies the risk bound%
\begin{equation*}
\mathbb{E}\left[ f\left( X\right) \right] \leq \max_{i=1}^{n-\tau }f\left(
X_{i}\right) +\frac{2L}{n-\tau }\sum_{k=\tau +1}^{n}\min_{i=1}^{k-\tau
}d\left( X_{k},X_{i}\right) +\varphi \left( \tau \right) +\text{diam}\left( 
\mathcal{X}\right) \frac{e\ln \left( 1/\delta \right) }{n-\tau }.
\end{equation*}%
The left hand side is the risk of $f$ in the invariant distribution,
independent of the sample path. The maximum in the first term on the right
hand side, instead of the customary average, may be disturbing, but whenever
the learning algorithm achieves near zero training error, it shouldn't make
much difference. We will nevertheless provide a mechanism to handle a small
fraction of outliers or corrupted data points. The next term estimates the
expected distance of data from the sample path, scaled with the Lipschitz
constant $L$. This estimate, which is easily computed from the sample,
decides over success and failure of the proposed bound. $\varphi \left( \tau
\right) $ is the mixing coefficient. It is the penalty for regarding
observations, which are $\tau $ units of time apart, as independent. For iid
processes we set the mixing time $\tau $ to $1$ and $\varphi \left( \tau
\right) $ to zero. 

There is no multiplicative dependence on $\tau $, nor is there any mention
of dimension or numbers of parameters, and it will be shown below, that the
constraint on the Lipschitz norm of the functions can be removed and
replaced by local Lipschitz properties in neighborhoods of the sample points 
$X_{i}$.

In summary the contribution of this paper is a data-dependent generalization
guarantee, which is uniform over large or unconstrained hypothesis classes.
It is without multiplicative dependence on the mixing time, but requires
small training error and an effective simplicity of the data distribution,
as witnessed by the bound itself.

\section{Notation and preliminaries\label{Section preliminaries}}

We use capital letters for random variables, bold letters for vectors, and
the set $\left\{ 1,...,m\right\} $ is denoted $\left[ m\right] $. The
cardinality, complement and indicator function of a set $A$ are denoted $%
\left\vert A\right\vert $, $A^{c}$ and $\mathbf{1}A$ respectively. For a
real-valued function $f$ on a set $A$ we write$\left\Vert f\right\Vert
_{\infty }=\sup_{z\in A}\left\vert f\left( z\right) \right\vert $ and if $%
\mathcal{F}$ is a set of functions $f:A\rightarrow \mathbb{R}$ we write $%
\left\Vert \mathcal{F}\right\Vert _{\infty }=\sup_{f\in \mathcal{F}%
}\left\Vert f\right\Vert _{\infty }$.

Throughout the following $\left( \mathcal{X},\sigma \right) $ is a
measurable space and $\mathbf{X}=\left( X_{i}\right) _{i\in \mathbb{N}}$ a
stochastic process with values in $\mathcal{X}$. For $I\subseteq \mathbb{N}$%
, $\sigma \left( I\right) $ denotes the sigma-field generated by $\left(
X_{i}\right) _{i\in I}$ and $\mu _{I}$ the corresponding joint marginal. The
process is assumed to be stationary, so that $\forall I\subseteq \mathbb{Z}$%
, $i\in \mathbb{N}$, $\mu _{I}=\mu _{I+i}$, the stationary distribution
being denoted $\pi =\mu _{\left\{ 0\right\} }=\mu _{\left\{ k\right\} }$. It
is called ergodic if for every $A\in \sigma $ with $\pi \left( A\right) >0$
we have $\Pr \left\{ \forall k\leq n,X_{k}\notin A\right\} \rightarrow 0$ as 
$n\rightarrow \infty $. The $\varphi $-mixing and $\alpha $-mixing
coefficients (\cite{yu1994rates}, \cite{bradley2005basic}) are defined for
any $\tau \in \mathbb{N}$ as 
\begin{eqnarray*}
	\varphi \left( \tau \right)  &=&\sup \left\{ \left\vert \Pr \left(
	A|B\right) -\Pr A\right\vert :k\in \mathbb{Z},A\in \sigma \left( \left\{
	k\right\} \right) ,B\in \sigma \left( \left\{ i:i<k-\tau \right\} \right)
	\right\} , \\
	\alpha \left( \tau \right)  &=&\sup \left\{ \left\vert \Pr \left( A\cap
	B\right) -\Pr A\Pr B\right\vert :k\in \mathbb{Z},A\in \sigma \left( \left\{
	k\right\} \right) ,B\in \sigma \left( \left\{ i:i<k-\tau \right\} \right)
	\right\} .
\end{eqnarray*}%
The process is called $\varphi $-mixing ($\alpha $-mixing) if $\varphi
\left( \tau \right) \rightarrow 0$ ($\alpha \left( \tau \right) \rightarrow 0
$) as $\tau \rightarrow \infty $. When the mixing coefficients decrease
exponentially, then there is a characteristic mixing time $\tau _{\varphi }$
or $\tau _{\alpha }$ such that $\varphi \left( \tau _{\varphi }\right)
=\alpha \left( \tau _{\alpha }\right) =1/4$. We will refer to this as $\tau $%
, an imprecision justified by the difficulty to ascertain the mixing times
in general. For an iid process $\varphi \left( \tau \right) =\alpha \left(
\tau \right) =0$ for all $\tau \in \mathbb{N}$.

A \textit{loss class} $\mathcal{F}$ is a set of measurable functions $f:%
\mathcal{X}\rightarrow \left[ 0,\infty \right) $, where $f$ is to be thought
of as a hypothesis composed with a fixed loss function. Very often $\mathcal{%
	X=Z\times Z}^{\prime }$, where $\mathcal{Z}$ is a measurable space of
"inputs", $\mathcal{Z}^{\prime }$ is a space of "outputs", "covariates" or
"labels", $\mathcal{H}$ is a set of functions $h:\mathcal{Z}\rightarrow 
\mathbb{R}$ and $\ell $ is a fixed loss function $\ell :\mathbb{R\times }%
\mathcal{Z}^{\prime }\rightarrow \left[ 0,\infty \right) $. The loss class
in question would then be the class of functions%
\begin{equation*}
\mathcal{F}=\left\{ \left( z,z^{\prime }\right) \mapsto \ell \left( h\left(
z\right) ,z^{\prime }\right) :h\in \mathcal{H}\right\} .
\end{equation*}%
We will use the usual total ordering and topology on the extended real
number system.

\section{Gauge pairs and generalization\label{Section main theorem}}

\begin{definition}
	Let $\mathcal{F}$ be a class of nonnegative loss functions on a space $%
	\mathcal{X}$. A gauge pair for $\mathcal{F}$ is a pair $\left( g,\Phi
	\right) $ where $g$ is a measurable function $g:\mathcal{X\times X}%
	\rightarrow \left[ 0,\infty \right] $ such that $g\left( x,y\right) =0$ iff $%
	x=y$, and $\Phi $ is a function $\Phi :\mathcal{F}\times \mathcal{X}%
	\rightarrow \left[ 0,\infty \right] \mathcal{\ }$such that for all $x,y\in 
	\mathcal{X}$ and $f\in \mathcal{F}$, $\Phi \left( f,\cdot \right) $ is
	measurable and%
	\begin{equation*}
	f\left( y\right) \leq g\left( y,x\right) +\Phi \left( f,x\right) .
	\end{equation*}
\end{definition}

Intuitively $g\left( y,x\right) $ should measure the extent to which members
of $\mathcal{F}$ can generalize from an observed datum $x$ to a yet
unobserved datum $y$. It helps to think of $g$ as a nondecreasing function
of a metric, but $g$ need not be symmetric. 

\textbf{Example 1.} Let $\mathcal{F}$ be a class of $L$-Lipschitz functions
on a metric space $\left( \mathcal{X},d\right) $. This means that $f\left(
y\right) -f\left( x\right) \leq Ld\left( y,x\right) $ for all $x,y\in 
\mathcal{X}$ and $f\in \mathcal{F}$. Adding $f\left( x\right) $ to this
inequality shows that $\left( g:\left( y,x\right) \mapsto Ld\left(
y,x\right) ,\Phi :\left( f,x\right) \mapsto f\left( x\right) \right) $ is a
gauge pair for $\mathcal{F}$. 

\textbf{Example 2. }Function estimation. Let $\mathcal{X=Z}\times \mathbb{R}$
and $\mathcal{H}:\mathcal{Z}\rightarrow \mathbb{R}$ a set of functions of
Lipschitz norm at most $L$. With absolute loss we have the loss class $%
\mathcal{F}=\left\{ \left( z,z^{\prime }\right) \in \mathcal{Z}\times 
\mathbb{R}\mapsto \left\vert h\left( z\right) -z^{\prime }\right\vert :h\in 
\mathcal{H}\right\} $. Then $g:\left( \left( y,y^{\prime }\right) ,\left(
x,x^{\prime }\right) \right) \mapsto Ld\left( y,x\right) +\left\vert
y^{\prime }-x^{\prime }\right\vert $ and $\Phi :\left( f,x\right) \mapsto
f\left( x\right) $ are a gauge pair for $\mathcal{F}$.

\textbf{Example 3. }For binary classification, with $\mathcal{H}$ as above, $%
\mathcal{X=Z}\times \left\{ -1,1\right\} $ and the hinge loss, the loss
class becomes $\mathcal{F}=\left\{ \left( z,z^{\prime }\right) \mapsto
\left( 1-z^{\prime }h\left( z\right) \right) _{+}:h\in \mathcal{H}\right\} $%
. Define $g\left( \left( y,y^{\prime }\right) ,\left( x,x^{\prime }\right)
\right) =Ld\left( y,x\right) $ if $y^{\prime }=x^{\prime }$ and $g\left(
\left( y,y^{\prime }\right) ,\left( x,x^{\prime }\right) \right) =+\infty $
otherwise, and verify that the $g$ so defined and $\Phi :\left( f,x\right)
\mapsto f\left( x\right) $ are a gauge pair for $\mathcal{F}$.\bigskip 

The next section will give more examples of gauge pairs and show how they
can be defined to give bounds for other loss classes. Here we state and
prove our main result. The first part bounds the probability of excessive
losses, the other part bounds the risk properly.

\begin{theorem}
	\label{Theorem Main}Let $\mathbf{X}=\left( X_{i}\right) _{i\in \mathbb{N}}$
	be a stationary process with values in $\mathcal{X}$ and invariant
	distribution $\pi $, $\mathcal{F}$ a class of measurable functions $f:%
	\mathcal{X}\rightarrow \left[ 0,\infty \right) $ and $\left( g,\Phi \right) $
	a gauge pair for $\mathcal{F}$. Let $X\sim \pi $ be independent of $\mathbf{X%
	}$, $\tau \in \mathbb{N}$, $n>\tau $, $B\subseteq \left[ n-\tau \right] $
	and $\delta >0$. 
	
	(i) For any $t>0$, with probability at least $1-\delta $ in the sample path $%
	\mathbf{X}_{1}^{n}=\left( X_{1},...,X_{n}\right) $%
	\begin{multline*}
	\sup_{f\in \mathcal{F}}\Pr \left\{ f\left( X\right) >\max_{j\in B^{c}\cap %
		\left[ n-\tau \right] }\Phi \left( f,X_{j}\right) +t~|~\mathbf{X}%
	_{1}^{n}\right\}  \\
	\leq \frac{2}{n-\tau }\sum_{k=\tau +1}^{n}\mathbf{1}\left\{ \min_{i\in
		B^{c}\cap \left[ k-\tau \right] }g\left( X_{k},X_{i}\right) >t\right\}
	+\varphi \left( \tau \right) +\frac{e\ln \left( 1/\delta \right) }{n-\tau }.
	\end{multline*}
	
	(ii) With probability at least $1-\delta $ in the sample path 
	\begin{multline*}
	\sup_{f\in \mathcal{F}}\mathbb{E}\left[ f\left( X\right) \right] -\max_{i\in
		B^{c}\cap \left[ n-\tau \right] }\Phi \left( f,X_{i}\right)  \\
	\leq \frac{2}{n-\tau }\sum_{k=\tau +1}^{n}\min_{i\in B^{c}\cap \left[ n-\tau %
		\right] }g\left( X_{k},X_{i}\right) +\left\Vert \mathcal{F}\right\Vert
	_{\infty }\varphi _{\tau }+\left\Vert g\right\Vert _{\infty }\frac{e\ln
		\left( 1/\delta \right) }{n-\tau }.
	\end{multline*}
\end{theorem}

\textbf{Remark 1.} The bound in the introduction follows from (ii) and the
Lipschitz gauge, $g:\left( x,y\right) \mapsto Ld\left( x,y\right) $ and $%
\Phi :\left( f,x\right) \mapsto f\left( x\right) $.

\textbf{Remark 2.} Both (i) and (ii) have two data-dependent terms, the
first of which depends on the choice of the function $f$. It is the term,
which a learning algorithm should try to minimize. In the three examples
above it depends only on the evaluation of $f$ at the sample points, but in
general it may depend on any local properties of $f$, as shown in the next
section. In part (i), which limits the probability of excessive losses, the
appearance of the maximum is perhaps more natural. Notice that part (i)
makes sense also for unbounded loss classes and potentially infinite-valued $%
g$.

\textbf{Remark 3}. The presence of the maximum makes the bound sensitive to
outliers, but a union bound over the possible "bad" sets $B$ of bounded
cardinality can allow for a small fraction of errors. The method is standard
and detailed in the Section \ref{Section softening the maximum}. The set $B$
has been introduced solely for this purpose. In any case the bounds are
intended to target the case were zero, or near zero training error can be
achieved. In the case of binary classification, in Example 3 above, zero
training error would require a minimal distance of $2/L$ between sample
points with different labels, thus amounting to a hard margin condition.

\textbf{Remark 4.} The terms 
\begin{equation*}
G_{t}=\frac{1}{n-\tau }\sum_{k=\tau +1}^{n}\mathbf{1}\left\{ \min_{i\in
	B^{c}\cap \left[ k-\tau \right] }g\left( X_{k},X_{i}\right) >t\right\} \text{
	or }G=\frac{1}{n-\tau }\sum_{k=\tau +1}^{n}\min_{i\in B^{c}\cap \left[
	k-\tau \right] }g\left( X_{k},X_{i}\right) 
\end{equation*}%
can be computed from the data and interpreted as a complexity of the sample
path relative to $g$. In Section \ref{Section Decay of Data Complexity} it
will be proven that, under an assumption of total boundedness of the support
of $\pi $ (defined relative to $g$) and ergodicity of the process, $G$
converges to zero in probability. For sufficiently fast $\alpha $-mixing
(and therefore also for iid processes) it converges to zero almost surely.
Unfortunately the worst case rate scales with covering numbers of the
support of the invariant distribution, which is why the approach is limited
to essentially low dimensional data. This also highlights the importance of
being able to directly observe these quantities. If $g$ is a metric, then $%
G_{t}$ is an upper estimate on the $t$-\textit{missing mass} (\cite%
{berend2012missing}, \cite{maurer2022concentration}), 
\begin{equation*}
\hat{M}\left( \mathbf{X,}t\right) :=\Pr \left\{ X:\min_{i=1}^{n-\tau
}g\left( X,X_{i}\right) >t\text{ }|\text{ }\mathbf{X}\right\} ,
\end{equation*}%
the probability of data outside the union of $t$-balls about the sample
points. The quantity $G$, obtained from $G_{t}$ by integration, upper bounds
the expected distance of data from the sample path. Even if $g$ is not a
metric we call $G_{t}$ and $G$ the \textit{missing mass estimates}.
Convergence and other properties of these quantities are further discussed
in Section \ref{Section Decay of Data Complexity}.

\textbf{Remark 5.} The quantitative mixing properties are difficult to
ascertain even for finite state Markov chains, where current estimates of
the mixing time $\tau $ have a sample complexity depending on the mixing
time itself (\cite{hsu2015mixing}, \cite{wolfer2019estimating}). For
continuous state spaces we are limited to educated guesses. These
uncertainties affect the stability of sample complexities $n\left( \epsilon
,\tau \right) $ of generalization bounds for weakly dependent processes,
where $\epsilon $ is the desired error level. For conventional bounds of $%
O\left( \tau /n\right) $ we have $\partial n\left( \epsilon ,\tau \right)
/\partial \tau =O\left( 1/\epsilon \right) $, whereas for the bound in
Theorem \ref{Theorem Main} $\partial n\left( \epsilon ,\tau \right)
/\partial \tau =O\left( 1\right) $, an obvious advantage.\bigskip 

The proof of Theorem \ref{Theorem Main} requires the following tail bound
for martingale difference sequences, with proof given in Section \ref{Proof
	of Lemma Martingale}.

\begin{lemma}
	\label{Lemma Martingale}Let $R_{1},...,R_{n}$ be real random variables $%
	0\leq R_{j}\leq 1$ and let $\sigma _{1}\subseteq \sigma _{2}\subseteq
	...\sigma _{n}$ be a filtration such that $R_{j}$ is $\sigma _{j}$%
	-measurable. Let $\hat{V}=\frac{1}{n}\sum_{j}R_{j}$, $V=\frac{1}{n}\sum_{j}%
	\mathbb{E}\left[ R_{j}|\sigma _{j-1}\right] $. Then%
	\begin{equation*}
	\Pr \left\{ V>2\hat{V}+t\right\} \leq e^{-nt/e}
	\end{equation*}%
	and equivalently, for $\delta >0$,%
	\begin{equation*}
	\Pr \left\{ V>2\hat{V}+\frac{e\ln \left( 1/\delta \right) }{n}\right\} \leq
	\delta .
	\end{equation*}
\end{lemma}

\begin{proof}[Proof of Theorem \protect\ref{Theorem Main}]
	Let $f\in \mathcal{F}$. From stationarity we obtain for every $k$, $\tau
	<k\leq n$ and $u>0$%
	\begin{equation*}
	\Pr \left\{ f\left( X\right) >u\right\} =\Pr_{X_{k}\sim \mu _{k}}f\left(
	X_{k}\right) >u\leq \Pr \left\{ f\left( X_{k}\right) >u~|~\left(
	X_{i}\right) _{i\in \left[ k-\tau \right] }\right\} +\varphi \left( \tau
	\right) ,
	\end{equation*}%
	where the inequality follows from the definition of the $\varphi $-mixing
	coefficients, which allows us to replace the independent variable $X$ by the
	sample path observation $X_{k}$, conditioned on observations more than $\tau 
	$ time increments in the past. But if $f\left( X_{k}\right) >u$, then by the
	definition of gauge pairs we must have $g\left( X_{k},X_{i}\right) +\Phi
	\left( f,X_{i}\right) >u$, for every $i\in \left[ k-\tau \right] $, or,
	equivalently, $\min_{i\in \left[ k-\tau \right] }g\left( X_{k},X_{i}\right)
	+\Phi \left( f,X_{i}\right) >u$, which certainly implies $\min_{i\in
		B^{c}\cap \left[ k-\tau \right] }g\left( X_{k},X_{i}\right) +\max_{j\in
		B^{c}\cap \left[ n-\tau \right] }\Phi \left( f,X_{j}\right) >u$. Thus%
	\begin{equation*}
	\Pr \left\{ f\left( X\right) >u\right\} \leq \Pr \left\{ \min_{i\in
		B^{c}\cap \left[ k-\tau \right] }g\left( X_{k},X_{i}\right) +\max_{j\in
		B^{c}\cap \left[ n-\tau \right] }\Phi \left( f,X_{j}\right) >u~|~\left(
	X_{i}\right) _{i\in \left[ k-\tau \right] }\right\} +\varphi \left( \tau
	\right) .
	\end{equation*}%
	Averaging this inequality over all values of $k$, $\tau <k\leq n$, and a
	change of variables $t=u-\max_{j\in S\cap \left[ n-\tau \right] }\Phi \left(
	f,X_{j}\right) $ gives%
	\begin{align}
	& \Pr \left\{ f\left( X\right) >\max_{j\in B^{c}\cap \left[ n-\tau \right]
	}\Phi \left( f,X_{j}\right) +t~|~\mathbf{X}_{1}^{n}\right\} 
	\label{Big Inequality in Proof} \\
	& \leq \frac{1}{n-\tau }\sum_{k=\tau +1}^{n}\Pr \left\{ \min_{i\in B^{c}\cap %
		\left[ k-\tau \right] }g\left( X_{k},X_{i}\right) >t~|~\left( X_{i}\right)
	_{i\in \left[ k-\tau \right] }\right\} +\varphi \left( \tau \right) .  \notag
	\end{align}%
	Notice that the right hand side of this inequality is independent of $f$.
	
	We can now prove (i). Let $\sigma _{k}$ be the $\sigma $-algebra generated
	by $\left( X_{i}\right) _{i\in \left[ k\right] }$ and $R_{k}=\mathbf{1}%
	\left\{ \min_{i\in B^{c}\cap \left[ k-\tau \right] }g\left(
	X_{k},X_{i}\right) >t\right\} $, so that $R_{k}$ is $\sigma _{k}$%
	-measurable. Then $\mathbb{E}\left[ R_{k}|\sigma _{k-1}\right] =\Pr \left\{
	\min_{i\in B^{c}\cap \left[ k-\tau \right] }g\left( X_{k},X_{i}\right)
	>t~|~\left( X_{i}\right) _{i\in \left[ k-\tau \right] }\right\} $ and from
	Lemma \ref{Lemma Martingale} we obtain with probability at least $1-\delta $ 
	\begin{multline*}
	\frac{1}{n-\tau }\sum_{k=\tau +1}^{n}\Pr \left\{ \min_{i\in B^{c}\cap \left[
		k-\tau \right] }g\left( X_{k},X_{i}\right) >t~|~\left( X_{i}\right) _{i\in %
		\left[ k-\tau \right] }\right\}  \\
	\leq \frac{1}{n-\tau }\sum_{k=\tau +1}^{n}\mathbf{1}\left\{ \min_{i\in
		B^{c}\cap \left[ k-\tau \right] }g\left( X_{k},X_{i}\right) >t\right\} +%
	\frac{e\ln \left( 1/\delta \right) }{n-\tau }.
	\end{multline*}%
	Substitute this in the right hand side of (\ref{Big Inequality in Proof}).
	As the the right hand side is independent of $f$, we can take the supremum
	over $f$ on the left hand side to complete the proof of (i).
	
	For (ii) we use integration by parts and integrate the left hand side of (%
	\ref{Big Inequality in Proof}) from zero to $\left\Vert \mathcal{F}%
	\right\Vert _{\infty }$. This gives%
	\begin{align*}
	& \mathbb{E}\left[ f\left( X\right) \right] -\max_{j\in B^{c}\cap \left[
		n-\tau \right] }\Phi \left( f,X_{j}\right)  \\
	& \leq \frac{1}{n-\tau }\sum_{k=\tau +1}^{n}\int_{0}^{\left\Vert \mathcal{F}%
		\right\Vert _{\infty }}\mathbb{E}\left[ \mathbf{1}\left\{ \min_{i\in
		B^{c}\cap \left[ k-\tau \right] }g\left( X_{k},X_{i}\right) >t\right\}
	~|~\left( X_{i}\right) _{i\in \left[ k-\tau \right] }\right] dt+\left\Vert 
	\mathcal{F}\right\Vert _{\infty }\varphi \left( \tau \right)  \\
	& =\frac{1}{n-\tau }\sum_{k=\tau +1}^{n}\mathbb{E}\left[ \min_{i\in
		B^{c}\cap \left[ k-\tau \right] }g\left( X_{k},X_{i}\right) ~|~\left(
	X_{i}\right) _{i\in \left[ k-\tau \right] }\right] +\left\Vert \mathcal{F}%
	\right\Vert _{\infty }\varphi \left( \tau \right) .
	\end{align*}%
	Now we use Lemma \ref{Lemma Martingale} just as in part (i). Again we can
	take the supremum on the left hand side. This completes the proof of (ii).
\end{proof}

Instead of Lemma \ref{Lemma Martingale} the Hoeffding-Azuma inequality \cite%
{McDiarmid98} could also be used in the proof. This would lead to $O\left(
n^{-1/2}\right) $-rates in the last terms and a constant factor of $1$
instead of $2$ for $G$ and $G_{t}$, which might be preferable if $G$ and $%
G_{t}$ are the dominant terms.

\section{Gauge pairs and unconstrained classes\label{Section examples of
		gauge pairs}}

In the previous section we gave the example of Lipschitz classes, where the
function $g$ was proportional to a metric, and $\Phi $ was the evaluation
functional. Here we briefly discuss smooth classes and then show that
generalization bounds are possible even for completely unconstrained loss
classes.

\textbf{Smooth functions.} A real-valued, differentiable function on an open
subset $\mathcal{O}$ of a Hilbert space is called $\gamma $-smooth if $%
\gamma >0$ and for all $x,y\in \mathcal{X}$%
\begin{equation*}
\left\Vert f^{\prime }\left( y\right) -f^{\prime }\left( x\right)
\right\Vert \leq \gamma \left\Vert y-x\right\Vert .
\end{equation*}%
$\gamma $-smoothness plays a role in non-convex optimization because of the
descent-lemma, giving a justification to the method of gradient descent (%
\cite{bubeck2015convex}). Nonnegative $\gamma $-smooth functions satisfy
special inequalities as in the following lemma, with proof in Section \ref%
{Proof of Lemma gamma smooth}.

\begin{lemma}
	\label{Lemma gamma smooth}Let $f$ be a nonnegative, differentiable, $\gamma $%
	-smooth function on a convex open subset $\mathcal{O}$ of a Hilbert space
	and $\lambda >0$. Then for any $x,y\in \mathcal{O}$%
	\begin{equation*}
	f\left( y\right) \leq \left( 1+\lambda ^{-1}\right) f\left( x\right) +\left(
	1+\lambda \right) \frac{\gamma }{2}\left\Vert y-x\right\Vert ^{2}.
	\end{equation*}%
	If $f\left( x\right) =0$ then $f\left( y\right) \leq \frac{\gamma }{2}%
	\left\Vert y-x\right\Vert ^{2}$.
\end{lemma}

So $g:\left( y,x\right) \mapsto \left( 1+\lambda \right) \frac{\gamma }{2}%
\left\Vert y-x\right\Vert ^{2}$ and $\left( f,x\right) \mapsto \left(
1+\lambda ^{-1}\right) f\left( x\right) $ make a gauge pair for the class of 
$\gamma $-smooth functions on $\mathcal{O}$, to which Theorem \ref{Theorem
	Main} can be applied.

The passage from the metric $g\approx \left\Vert y-x\right\Vert $ to its
square $g\approx \left\Vert y-x\right\Vert ^{2}$ will improve the
data-complexity term $G$, by decreasing the covering numbers as they appear
in Theorem \ref{Theorem decay of G} in the next section. If $\mathcal{N}%
\left( \text{supp}\left( \pi \right) ,g,t\right) $ is the number of sets $C$
with $\sup_{x,y\in C}g\left( x,y\right) <t$ which are necessary to cover the
support of $\pi $, then $\mathcal{N}\left( \text{supp}\left( \pi \right)
,\left\Vert \cdot -\cdot \right\Vert ^{p},t\right) =\mathcal{N}\left( \pi ,%
\text{supp}\left( \pi \right) ,\left\Vert \cdot -\cdot \right\Vert
,t^{1/p}\right) $. Even if $\pi $ has full support on the unit ball of a $D$%
-dimensional Banach space, and $g$ is the euclidean metric, then $\mathcal{N}%
\left( \text{supp}\left( \pi \right) ,\left\Vert \cdot -\cdot \right\Vert
,t\right) =Kt^{-D}$ (\cite{Cucker01}) and $\mathcal{N}\left( \text{supp}%
\left( \pi \right) ,\left\Vert \cdot -\cdot \right\Vert ^{p},t\right)
=Kt^{-D/p}$, decreasing the covering number by the factor $t^{D\left(
	1-1/p\right) }$ which can make a big difference already for $p=2$.

\textbf{Local Lipschitz properties.} If we know that most of the mass of the
distribution is contained in a neighborhood of the sample path, then
regularity properties of a loss function outside of this neighborhood should
be largely irrelevant. Let $\left( \mathcal{X},d\right) $ be a metric space, 
$\mathcal{F}$ the class of \textit{all measurable functions} $f:\mathcal{%
	X\rightarrow }\left[ 0,\infty \right) $, and define $L:\mathcal{F\times X}%
\times \left( 0,\infty \right) \rightarrow \left[ 0,\infty \right) $ by%
\begin{equation*}
L\left( f,x,r\right) :=\sup_{y:0<d\left( y,x\right) \leq r}\frac{f\left(
	y\right) -f\left( x\right) }{d\left( y,x\right) }.
\end{equation*}%
$L\left( f,x,r\right) $ is a Lipschitz property of $f$ relative to $x$ at
scale $r$. It is always bounded by the Lipschitz constant of $f$ restricted
to the ball of radius $r$ about $x$ (sometimes the latter is called the 
\textit{local Lipschitz constant }(\cite{jordan2020exactly})), but it may be
substantially smaller (take $f\left( x_{1},x_{2}\right) =$ sign$\left(
x_{1}\right) \sqrt{x_{1}^{2}+x_{2}^{2}}$ in $\mathbb{R}^{2}$. Then $L\left(
f,\left( 0,0\right) ,r\right) =1$ for every $r>0$, but the local Lipschitz
constant is infinite on any ball of nonzero radius about $\left( 0,0\right) $%
).

Let $\rho :\left( 0,\infty \right) \rightarrow \left( 0,\infty \right] $ be
non-decreasing. Then for $y,x\in \mathcal{X}$, conjugate exponents $p$ and $q
$, and any $f\in \mathcal{F}$, with Young's inequality,%
\begin{eqnarray*}
	f\left( y\right)  &\leq &f\left( x\right) +\sup_{r>0}\frac{L\left(
		f,x,r\right) }{\rho \left( r\right) }\rho \left( d\left( y,x\right) \right)
	d\left( y,x\right)  \\
	&\leq &f\left( x\right) +\frac{1}{p}\left( \sup_{r>0}\frac{L\left(
		f,x,r\right) }{\rho \left( r\right) }\right) ^{p}+\frac{1}{q}\left( \rho
	\left( d\left( y,x\right) \right) d\left( y,x\right) \right) ^{q}.
\end{eqnarray*}%
It follows that 
\begin{equation*}
\Phi \left( f,x\right) =f\left( x\right) +\frac{1}{p}\left( \sup_{r>0}\frac{%
	L\left( f,x,r\right) }{\rho \left( r\right) }\right) ^{p}\text{ and }g\left(
y,x\right) =\frac{1}{q}\left( \rho \left( d\left( y,x\right) \right) d\left(
y,x\right) \right) ^{q}
\end{equation*}%
define a gauge pair, to which Theorem \ref{Theorem Main} (i) can be applied
and gives a high probability bound, uniform over all measurable non-negative
functions. When can we expect this bound to be finite or even small?

Let $r_{0}>0$ and define $\rho \left( r\right) =1$ if $r\leq r_{0}$ and $%
\rho \left( r\right) =+\infty $ if $r>r_{0}$, so with $p=q=2$ we get $\Phi
\left( f,x\right) =f\left( x\right) +L\left( f,x,r_{0}\right) ^{2}/2$ and $%
g\left( y,x\right) =d\left( y,x\right) $ if $d\left( y,x\right) \leq r_{0}$
and $+\infty $ otherwise. With $\mathcal{F}$ being the set of all
non-negative measurable functions on $\mathcal{X}$ the bound in Theorem \ref%
{Theorem Main} (i) (for simplicity in the iid case) then says, that with
probability at least $1-\delta $%
\begin{multline*}
\sup_{f\in \mathcal{F}}\Pr \left\{ f\left( X\right) >\max_{j\in \left[ n-1%
	\right] }f\left( X_{j}\right) +\frac{L\left( f,X_{j},r_{0}\right) ^{2}}{2}%
+t~|~\mathbf{X}_{1}^{n}\right\}  \\
\leq \frac{2}{n-1}\sum_{k=2}^{n}\mathbf{1}\left\{ \min_{i\in \left[ k-1%
	\right] }d\left( X_{k},X_{i}\right) ^{2}>\min \left\{ t,r_{0}^{2}\right\}
\right\} +\frac{e\ln \left( 1/\delta \right) }{n-1}.
\end{multline*}%
Instead of global properties of the members of $\mathcal{F}$ we require that
the chosen function $f$ be nearly flat in the neighborhoods of the $X_{j}$.
More than $r_{0}$ away from the sample path the function $f$ can be
arbitrarily wild and discontinuous. Flatness of the function at the training
points in the input space has been considered as a condition to enhance
adversarial robustness and generalization by \cite{stutz2021relating}. If
there aren't too many indices $j$, where $f\left( X_{j}\right) +L\left(
f,X_{j},r_{0}\right) ^{2}/2$ is large, these indices can be collected in the
exception set $B$. Clearly there is also a trade-off in the choice of $r_{0}$%
. If $r_{0}$ is too small, the missing mass estimate becomes too large.

Since the function $g$ so defined is unbounded, we can only apply Theorem %
\ref{Theorem Main} part (i) to the scheme described above. A remedy would be
to replace $+\infty $ by some large constant in the (e.g. $n^{1/4}$) in the
definition of $\rho $. More elegantly, if we restrict $\mathcal{F}$ to
nonnegative differentiable functions on a bounded, convex open set in a
Hilbert space, local $\gamma $-smoothness can be defined by 
\begin{equation*}
\gamma \left( f,x\right) =\sup_{y\neq x}\frac{\left\Vert f^{\prime }\left(
	y\right) -f^{\prime }\left( x\right) \right\Vert }{\left\Vert x-y\right\Vert 
}.
\end{equation*}%
Hence, as shown also in Section \ref{Proof of Lemma gamma smooth}, $L\left(
f,x,r\right) \leq \left\Vert f^{\prime }\left( x\right) \right\Vert +\gamma
\left( f,x\right) r/2$, and, with $\rho \left( r\right) =c\left(
1+r^{2}\right) $ for $c>0$, 
\begin{equation}
\sup_{r>0}\frac{L\left( f,x,r\right) }{\rho \left( r\right) }\leq
c^{-1}\left( \left\Vert f^{\prime }\left( x\right) \right\Vert +\gamma
\left( f,x\right) /4\right) .  \label{Local Gamma smooth}
\end{equation}%
Then both parts of Theorem \ref{Theorem Main} can be applied with the gauge
pair%
\begin{equation*}
\Phi \left( f,x\right) =f\left( x\right) +\frac{1}{2c}\left( \left\Vert
f^{\prime }\left( x\right) \right\Vert +\frac{\gamma \left( f,x\right) }{4}%
\right) ^{2}\text{ and }g\left( y,x\right) =\frac{1}{2}\left( c\left(
1+\left\Vert y-x\right\Vert \right) ^{2}\right) ^{2}\left\Vert
y-x\right\Vert ^{2}.
\end{equation*}%
Notice that $f\left( X_{j}\right) =0$ also forces $f^{\prime }\left(
X_{j}\right) =0$, so with zero training error the penalty depends only on
the local smoothness.

\section{The missing mass\label{Section Decay of Data Complexity}}

In this section we first address the convergence of the missing mass
estimates. Then we present a related estimator for the iid case, which
characterizes generalization for the full set of $L$-Lipschitz functions, by
providing both upper and lower empirical bounds.

For convergence we limit ourselves to the estimate in part (ii) of Theorem %
\ref{Theorem Main}%
\begin{equation*}
G\left( \mathbf{X},n,\tau ,g\right) =\frac{1}{n-\tau }\sum_{k=\tau
	+1}^{n}\min_{i\in \left[ k-\tau \right] }g\left( X_{k},X_{i}\right) ,
\end{equation*}%
where for simplicity we omit the set $B$. It is clearly a merit of Theorem %
\ref{Theorem Main}, that we can observe this quantity from the sample and
thus take advantage of favorable situations. To get some idea of what
characterizes these favorable situations, it is nevertheless interesting to
study the behavior of $G$ in general, although our quantitative worst-case
guarantees are disappointing.

Some standard concepts of the theory of metric spaces, such as diameters and
covering numbers, extend to the case, when the metric is replaced by the
function $g$.

\begin{definition}
	For $A\subseteq \mathcal{X}$ the $g$-diameter diam$_{g}\left( A\right) $ of $%
	A$ is diam$_{g}\left( A\right) =\sup_{x,y\in F}g\left( y,x\right) $. For $%
	A\subseteq \mathcal{X}$ and $\epsilon >0$ we write%
	\begin{equation*}
	\mathcal{N}\left( A,g,\epsilon \right) =\min \left\{ N:\exists
	C_{1},...,C_{N},\text{ diam}_{g}\left( C_{i}\right) \leq \epsilon \text{ and 
	}A\subseteq \bigcup_{j\in \left[ N\right] }C_{j}\right\} .
	\end{equation*}%
	$A\subseteq \mathcal{X}$ is $g$-totally bounded if $\mathcal{N}\left(
	A,g,\epsilon \right) <\infty $ for every $\epsilon >0$.
\end{definition}

The proof of the following theorem is given in Section \ref{Proof of Theorem
	decay of G}.

\begin{theorem}
	\label{Theorem decay of G}Let $\mathbf{X}=\left( X_{i}\right) _{i\in \mathbb{%
			N}}$ be a stationary process and assume the support of $\pi $ to be $g$%
	-totally bounded.
	
	(i) If $\mathbf{X}$ is ergodic, then for any $\tau \in \mathbb{N}$, we have $%
	G\left( \mathbf{X},n,\tau ,g\right) \rightarrow 0$ in probability as $%
	n\rightarrow \infty $
	
	(ii) If there exists $q>1$ and $A>0$ such that the $\alpha $-mixing
	coefficients satisfy $\alpha _{\tau }\leq A\tau ^{-q}$ for all $\tau \in 
	\mathbb{N}$, then for $p\in \left( 2/\left( 1+q\right) ,1\right) $ we have $%
	G\left( \mathbf{X},n,\left\lceil n^{p}\right\rceil ,g\right) \rightarrow 0$
	almost surely.
	
	(iii) In general we have%
	\begin{equation*}
	\Pr \left\{ G\left( \mathbf{X},n,\tau ,g\right) >t\right\} \leq \frac{%
		\mathcal{N}\left( \text{supp}\left( \pi \right) ,g,t/2\right) }{e\left(
		\left\lfloor nt/\left( 2\tau \right) \right\rfloor -1\right) }+\left\lceil 
	\frac{nt}{2\tau }\right\rceil \alpha _{\tau }.
	\end{equation*}
\end{theorem}

The basic intuition for the proof of (i) is that the ergodic process must
visit every member of a cover of the support of $\pi $, and by stationarity
it must in fact visit it infinitely often, which forces convergence in
probability. The other two conclusions require more technical
arguments.\bigskip 

While part (i) is encouraging, the worst case bound (iii) seems very weak
for two reasons.

1. It depends strongly on the mixing properties of the process. In
particular it scales with $\tau /n$, so as to exhibit a multiplicative
scaling of the sample complexity with the mixing time, and to refute one of
the claims made about this paper. We will address this problem in the next
section, and show that it is irrelevant in many practical cases.

2. It scales with the covering number of the support of $\pi $. This
behavior persists in the iid case, when $\tau =1$ and $\alpha _{\tau }=0$.
Let us accept this scaling at face value. We can expect the bound of Theorem %
\ref{Theorem Main} to be strong only if the support of the invariant
distribution is a small and essentially low dimensional object, even though
it may be nonlinearly embedded in a high- or even infinite dimensional space.

But many dynamical systems have relatively few degrees of freedom, even
though they may generate high dimensional data. Let $\Gamma $ be a compact
phase space with continuous metric $d$ and $T$ an $L$-Lipschitz map from $%
\Gamma $ to a high dimensional normed space with metric $\left\Vert \cdot
-\cdot \right\Vert $. Then $\mathcal{N}\left( T\left( \Gamma \right)
,\left\Vert \cdot -\cdot \right\Vert ,t\right) \leq \mathcal{N}\left( \Gamma
,d,t/L\right) $. So, if the phase space has few parameters, the first of
these covering numbers will be small independent of the dimension of the
ambient space of the data. Good examples are high resolution images of two-
or three-dimensional rotations of rigid bodies. In the supplement we report
experiments, where we measure $G$ with image rotations and combined
rotations and scalings, and, not surprisingly, observe a rapid decay as the
sample size increases.\bigskip 

We close this section with a brief discussion of the loss-class $\mathcal{F}$
of all $L$-Lipschitz functions in the iid case. Fix $t>0$ and note that%
\begin{equation*}
\sup_{f\in \mathcal{F}}\Pr \left\{ f\left( X\right) >\max_{i}f\left(
X_{i}\right) +t\text{ }|\text{ }\mathbf{X}\right\} =\Pr \left\{
\min_{i=1}^{n}d\left( X,X_{i}\right) >\frac{t}{L}\text{ }|\text{ }\mathbf{X}%
\right\} .
\end{equation*}%
This holds, since, whenever $f\left( X\right) >\max_{i}f\left( X_{i}\right)
+t$ and $f$ is $L$-Lipschitz, then $X$ must be further than $t/L$ from all
sample points. On the other hand, if $\min_{i=1}^{n}d\left( X,X_{i}\right)
>t/L$, then the function $\min_{i=1}^{n}d\left( X,X_{i}\right) $ witnesses
the supremum on the left hand side. The right hand side above is the $t/L$%
-missing mass $\hat{M}\left( \mathbf{X},t/L\right) $ introduced in Remark 4
of Section \ref{Section main theorem}. It can be estimated by the
generalized Good-Turing estimator (\cite{maurer2022concentration})%
\begin{equation*}
GT\left( \mathbf{X},t/L\right) =\frac{1}{n}\sum_{k=1}^{n}1\left\{
\min_{i:i\neq k}d\left( X_{k},X_{i}\right) >t/L\right\} .
\end{equation*}%
According to Theorem 2.1 (iv) in \cite{maurer2022concentration}, with proof
originally due to Sourav Chatterjee, we have 
\begin{equation}
\left\Vert \hat{M}\left( \mathbf{X},t/L\right) -GT\left( \mathbf{X}%
,t/L\right) \right\Vert _{2}\leq \sqrt{7/n},  \label{Good Turing bound}
\end{equation}%
from which Chebychev's inequality gives with probability at least $1-\delta $
in the sample 
\begin{equation}
\left\vert \sup_{f\in \mathcal{F}}\Pr \left\{ f\left( X\right)
>\max_{i}f\left( X_{i}\right) +t\text{ }|\text{ }\mathbf{X}\right\} -\frac{1%
}{n}\sum_{k=1}^{n}1\left\{ \min_{i:i\neq k}d\left( X_{k},X_{i}\right)
>t/L\right\} \right\vert \leq \sqrt{\frac{7}{n\delta }}\text{.}
\label{lower bound}
\end{equation}%
Notice the similarity of the Good-Turing estimator to the estimator $G_{t}$
used in Theorem \ref{Theorem Main} (i). Unfortunately no higher moment
bounds other than (\ref{Good Turing bound}) are available, so $G_{t}$ is
preferable also in the iid case, even though its bias is not controlled. The
data-dependent lower bound implied by (\ref{lower bound}) suggests, that the
class of all $L$-Lipschitz functions might be too big for good
generalization on effectively high dimensional data.

\section{Recurrence}

Theorem \ref{Theorem decay of G} (iii) gives a bound on the missing mass
estimate, which scales with the mixing time $\tau $. This is a worst case
bound, and the multiplicative scaling with $\tau $ is not generic. It is
clear from its definition, that the decay of $G$ depends on recurrence
rather than mixing. Kac (\cite{kac1947notion}) has shown, that for any
ergodic process the expected recurrence time of a set $A$ is $1/\pi \left(
A\right) $, without mixing assumptions. This implies that current
state-of-the-art upper estimates on mixing times for finite state ergodic
Markov chains (\cite{wolfer2019estimating}, Theorem 3) require sample sizes
at least as large as the largest expected recurrence time for any state.
Unfortunately the expectation appears to be the only moment of the
recurrence time, which can be controlled without mixing. For $\alpha $ and $%
\varphi $-mixing processes \cite{chazottes2003hitting} gives rapidly
decreasing tails of the recurrence time, but these would lead to similar
bounds as in Theorem \ref{Theorem decay of G} (iii).

While mixing implies ergodicity and recurrence, the converse does not hold.
A simple example is the deterministic unit rotation on the $N$-cycle, the
transition matrix being%
\begin{equation*}
P\left( i,j\right) =\delta _{i,\left( j+1\right) \text{ mod }N}.
\end{equation*}%
This Markov chain is ergodic and not mixing, but the recurrence time is $N$
and obviously $G\left( \mathbf{X},n,N,g\right) =0$ for all $n>N$ with $g$
being the discrete metric $g\left( y,x\right) :=1$ iff $y\neq x$ and $%
g\left( x,x\right) :=0$. In this form Theorem \ref{Theorem Main} does not
apply, but if we add some randomness, say%
\begin{equation*}
P\left( i,j\right) =\left( 1-p\right) \delta _{i,\left( j+1\right) \text{
		mod }N}+\frac{p}{N},\text{ for }p>0
\end{equation*}%
the process becomes exponentially mixing, with spectral gap $p$, relaxation
time $\tau _{rel}=p^{-1}$ and mixing time $\tau \geq p^{-1}\ln \left(
1/\epsilon \right) $ to have distance $\epsilon $ from stationarity (\cite%
{levin2017markov}). This gives an estimate of the mixing coefficients $%
\alpha _{\tau }\approx \exp \left( -\tau p\right) $. The probability that
the process visits all states in $N$ steps is $\left( 1-p\right) ^{N}$, and
the probability that it hasn't visited all states in $\tau $ steps is
bounded by $\left( 1-\left( 1-p\right) ^{N}\right) ^{\tau /N}$. As $%
p\rightarrow 0$ the mixing time diverges, but the recurrence times converges
to $N$. There are $2^{N}$ $\left\{ 0,1\right\} $-valued discrete Lipschitz
functions, so for large mixing times the dominant term in the classical
bounds (assuming the realizable case) would scale as $\tau \log \left(
2^{N}\right) /n\approx \tau N/n$, while Theorem \ref{Theorem Main} scales as 
$N/\left( n-\tau \right) $. 

This example may seem a bit trivial, but clearly a similar phenomenon is to
be expected for any periodic motion, but also irrational rotations on the
circle and more generally for any quasiperiodic motion, meaning that the
infinite trajectories are dense on the support of the invariant measure,
which may lie on a low-dimensional, but complicated submanifold of the
ambient space. \cite{arnold1968ergodic}\ have shown that classical dynamical
systems may be periodic, quasiperiodic or chaotic, and that the last two
cases are typically generic, with initial conditions having nonzero
phase--space-measure. Adding a perturbation can make a quasiperiodic system
mixing, but the mixing can be arbitrarily slow.

\section{Conclusion and limitations}

The proposed generalization bound for weakly dependent processes

\begin{itemize}
	\item applies uniformly over large or even unconstrained loss classes,
	
	\item has a sample complexity depending only additively on the mixing time,
	
	\item has a data-dependent complexity term and
	
	\item is independent of the dimension of the ambient space.
\end{itemize}

On the other hand it requires

\begin{itemize}
	\item a very small training error, except for very few sample points,
	
	\item an essentially low dimensional data distribution, and
	
	\item $\varphi $-mixing instead of the weaker $\alpha $- or $\beta $-mixing
	assumptions.
\end{itemize}

\bibliographystyle{abbrvnat}

\begin{thebibliography}{32}
	\providecommand{\natexlab}[1]{#1}
	\providecommand{\url}[1]{\texttt{#1}}
	\expandafter\ifx\csname urlstyle\endcsname\relax
	\providecommand{\doi}[1]{doi: #1}\else
	\providecommand{\doi}{doi: \begingroup \urlstyle{rm}\Url}\fi
	
	\bibitem[Agarwal and Duchi(2012)]{agarwal2012generalization}
	A.~Agarwal and J.~C. Duchi.
	\newblock The generalization ability of online algorithms for dependent data.
	\newblock \emph{IEEE Transactions on Information Theory}, 59\penalty0
	(1):\penalty0 573--587, 2012.
	
	\bibitem[Arnold and Avez(1968)]{arnold1968ergodic}
	V.~I. Arnold and A.~Avez.
	\newblock \emph{Ergodic problems of classical mechanics}, volume~9.
	\newblock Benjamin, 1968.
	
	\bibitem[Bartlett and Mendelson(2002)]{Bartlett02}
	P.~Bartlett and S.~Mendelson.
	\newblock Rademacher and gaussian complexities: Risk bounds and structural
	results.
	\newblock \emph{Journal of Machine Learning Research}, 3:\penalty0 463--482,
	2002.
	
	\bibitem[Bauer(2011)]{bauer2011probability}
	H.~Bauer.
	\newblock \emph{Probability theory}, volume~23.
	\newblock Walter de Gruyter, 2011.
	
	\bibitem[Berend and Kontorovich(2012)]{berend2012missing}
	D.~Berend and A.~Kontorovich.
	\newblock The missing mass problem.
	\newblock \emph{Statistics \& Probability Letters}, 82\penalty0 (6):\penalty0
	1102--1110, 2012.
	
	\bibitem[Berenfeld and Hoffmann(2021)]{berenfeld2021density}
	C.~Berenfeld and M.~Hoffmann.
	\newblock Density estimation on an unknown submanifold.
	\newblock \emph{Electronic Journal of Statistics}, 15\penalty0 (1):\penalty0
	2179--2223, 2021.
	
	\bibitem[Beygelzimer et~al.(2011)Beygelzimer, Langford, Li, Reyzin, and
	Schapire]{beygelzimer2011contextual}
	A.~Beygelzimer, J.~Langford, L.~Li, L.~Reyzin, and R.~Schapire.
	\newblock Contextual bandit algorithms with supervised learning guarantees.
	\newblock In \emph{Proceedings of the Fourteenth International Conference on
		Artificial Intelligence and Statistics}, pages 19--26. JMLR Workshop and
	Conference Proceedings, 2011.
	
	\bibitem[Bousquet and Elisseeff(2002)]{Bousquet02}
	O.~Bousquet and A.~Elisseeff.
	\newblock Stability and generalization.
	\newblock \emph{Journal of Machine Learning Research}, 2:\penalty0 499--526,
	2002.
	
	\bibitem[Bradley(2005)]{bradley2005basic}
	R.~C. Bradley.
	\newblock Basic properties of strong mixing conditions. a survey and some open
	questions.
	\newblock 2005.
	
	\bibitem[Bubeck et~al.(2015)]{bubeck2015convex}
	S.~Bubeck et~al.
	\newblock Convex optimization: Algorithms and complexity.
	\newblock \emph{Foundations and Trends{\textregistered} in Machine Learning},
	8\penalty0 (3-4):\penalty0 231--357, 2015.
	
	\bibitem[Chazottes(2003)]{chazottes2003hitting}
	J.~Chazottes.
	\newblock Hitting and returning to non-rare events in mixing dynamical systems.
	\newblock \emph{Nonlinearity}, 16\penalty0 (3):\penalty0 1017, 2003.
	
	\bibitem[Cucker and Smale(2001)]{Cucker01}
	F.~Cucker and S.~Smale.
	\newblock On the mathematical foundations of learning.
	\newblock \emph{Bulletin of the American Mathematical Society}, 39\penalty0
	(1):\penalty0 1--49, 2001.
	
	\bibitem[Fefferman et~al.(2016)Fefferman, Mitter, and
	Narayanan]{fefferman2016testing}
	C.~Fefferman, S.~Mitter, and H.~Narayanan.
	\newblock Testing the manifold hypothesis.
	\newblock \emph{Journal of the American Mathematical Society}, 29\penalty0
	(4):\penalty0 983--1049, 2016.
	
	\bibitem[Hsu et~al.(2015)Hsu, Kontorovich, and Szepesv{\'a}ri]{hsu2015mixing}
	D.~J. Hsu, A.~Kontorovich, and C.~Szepesv{\'a}ri.
	\newblock Mixing time estimation in reversible markov chains from a single
	sample path.
	\newblock \emph{Advances in neural information processing systems}, 28, 2015.
	
	\bibitem[Jordan and Dimakis(2020)]{jordan2020exactly}
	M.~Jordan and A.~G. Dimakis.
	\newblock Exactly computing the local lipschitz constant of relu networks.
	\newblock \emph{Advances in Neural Information Processing Systems},
	33:\penalty0 7344--7353, 2020.
	
	\bibitem[Kac(1947)]{kac1947notion}
	M.~Kac.
	\newblock On the notion of recurrence in discrete stochastic processes.
	\newblock 1947.
	
	\bibitem[Koltchinskii and Panchenko(2000)]{Koltchinskii00}
	V.~Koltchinskii and D.~Panchenko.
	\newblock Rademacher processes and bounding the risk of function learning.
	\newblock In J.~W. E.~Gine, D.~Mason, editor, \emph{High Dimensional
		Probability {II}}, pages 443--459. 2000.
	
	\bibitem[Levin and Peres(2017)]{levin2017markov}
	D.~A. Levin and Y.~Peres.
	\newblock \emph{Markov chains and mixing times}, volume 107.
	\newblock American Mathematical Soc., 2017.
	
	\bibitem[Ma and Fu(2012)]{ma2012manifold}
	Y.~Ma and Y.~Fu.
	\newblock \emph{Manifold learning theory and applications}, volume 434.
	\newblock CRC press Boca Raton, FL, 2012.
	
	\bibitem[Maurer(2022)]{maurer2022concentration}
	A.~Maurer.
	\newblock Concentration of the missing mass in metric spaces.
	\newblock \emph{arXiv preprint arXiv:2206.02012}, 2022.
	
	\bibitem[McAllester(1999)]{mcallester1999pac}
	D.~A. McAllester.
	\newblock Pac-bayesian model averaging.
	\newblock In \emph{Proceedings of the twelfth annual conference on
		Computational learning theory}, pages 164--170, 1999.
	
	\bibitem[McDiarmid(1998)]{McDiarmid98}
	C.~McDiarmid.
	\newblock Concentration.
	\newblock In \emph{Probabilistic Methods of Algorithmic Discrete Mathematics},
	pages 195--248, Berlin, 1998. Springer.
	
	\bibitem[Meir(2000)]{meir2000nonparametric}
	R.~Meir.
	\newblock Nonparametric time series prediction through adaptive model
	selection.
	\newblock \emph{Machine learning}, 39:\penalty0 5--34, 2000.
	
	\bibitem[Mohri and Rostamizadeh(2008)]{mohri2008rademacher}
	M.~Mohri and A.~Rostamizadeh.
	\newblock Rademacher complexity bounds for non-iid processes.
	\newblock \emph{Advances in Neural Information Processing Systems}, 21, 2008.
	
	\bibitem[Mohri and Rostamizadeh(2010)]{mohri2010stability}
	M.~Mohri and A.~Rostamizadeh.
	\newblock Stability bounds for stationary $\varphi$-mixing and $\beta$-mixing
	processes.
	\newblock \emph{Journal of Machine Learning Research}, 11\penalty0 (2), 2010.
	
	\bibitem[Shalizi and Kontorovich(2013)]{shalizi2013predictive}
	C.~Shalizi and A.~Kontorovich.
	\newblock Predictive pac learning and process decompositions.
	\newblock \emph{Advances in neural information processing systems}, 26, 2013.
	
	\bibitem[Steinwart and Christmann(2009)]{steinwart2009fast}
	I.~Steinwart and A.~Christmann.
	\newblock Fast learning from non-iid observations.
	\newblock \emph{Advances in neural information processing systems}, 22, 2009.
	
	\bibitem[Stutz et~al.(2021)Stutz, Hein, and Schiele]{stutz2021relating}
	D.~Stutz, M.~Hein, and B.~Schiele.
	\newblock Relating adversarially robust generalization to flat minima.
	\newblock In \emph{Proceedings of the IEEE/CVF International Conference on
		Computer Vision}, pages 7807--7817, 2021.
	
	\bibitem[Vapnik and Alexey(1971)]{vapnik1971chervonenkis}
	V.~N. Vapnik and Y.~Alexey.
	\newblock Chervonenkis. on the uniform convergence of relative frequencies of
	events to their probabilities.
	\newblock \emph{Theory of Probability and its Applications}, 16\penalty0
	(2):\penalty0 264--280, 1971.
	
	\bibitem[Wolfer and Kontorovich(2019{\natexlab{a}})]{wolfer2019estimating}
	G.~Wolfer and A.~Kontorovich.
	\newblock Estimating the mixing time of ergodic markov chains.
	\newblock In \emph{Conference on Learning Theory}, pages 3120--3159. PMLR,
	2019{\natexlab{a}}.
	
	\bibitem[Wolfer and Kontorovich(2019{\natexlab{b}})]{wolfer2019minimax}
	G.~Wolfer and A.~Kontorovich.
	\newblock Minimax learning of ergodic markov chains.
	\newblock In \emph{Algorithmic Learning Theory}, pages 904--930. PMLR,
	2019{\natexlab{b}}.
	
	\bibitem[Yu(1994)]{yu1994rates}
	B.~Yu.
	\newblock Rates of convergence for empirical processes of stationary mixing
	sequences.
	\newblock \emph{The Annals of Probability}, pages 94--116, 1994.
	
\end{thebibliography}

\appendix

\section{Remaining proofs\label{Section remaining Proofs}}

\subsection{A martingale lemma \label{Proof of Lemma Martingale}}

The idea to the following proof of Lemma \ref{Lemma Martingale} is taken
from Theorem 1 in \cite{beygelzimer2011contextual}.

\begin{lemma}[Lemma \protect\ref{Lemma Martingale} re-stated]
	Let $R_{1},...,R_{n}$ be real random variables $0\leq R_{j}\leq 1$ and let $%
	\sigma _{1}\subseteq \sigma _{2}\subseteq ...\sigma _{n}$ be a filtration
	such that $R_{j}$ is $\sigma _{j}$-measurable. Let $\hat{V}=\frac{1}{n}%
	\sum_{j}R_{j}$, $V=\frac{1}{n}\sum_{j}\mathbb{E}\left[ R_{j}|\sigma _{j-1}%
	\right] $. Then%
	\begin{equation*}
	\Pr \left\{ V>2\hat{V}+t\right\} \leq e^{-nt/e}
	\end{equation*}%
	and equivalently, for $\delta >0$,%
	\begin{equation*}
	\Pr \left\{ V>2\hat{V}+\frac{e\ln \left( 1/\delta \right) }{n}\right\} \leq
	\delta .
	\end{equation*}
\end{lemma}

\begin{proof}
	Let $Y_{j}:=\frac{1}{n}\left( \mathbb{E}\left[ R_{j}|\sigma _{j-1}\right]
	-R_{j}\right) $, so $\mathbb{E}\left[ Y_{j}|\sigma _{j-1}\right] =0$.\newline
	Then $\mathbb{E}\left[ Y_{j}^{2}|\sigma _{j-1}\right] =\left( 1/n^{2}\right)
	\left( \mathbb{E}\left[ R_{j}^{2}|\sigma _{j-1}\right] -\mathbb{E}\left[
	R_{j}|\sigma _{j-1}\right] ^{2}\right) \leq \left( 1/n\right) ^{2}\mathbb{E}%
	\left[ R_{j}|\sigma _{j-1}\right] $, since $0\leq R_{j}\leq 1$. Define a
	real function $g$ by%
	\begin{equation*}
	g\left( t\right) =\frac{e^{t}-t-1}{t^{2}}\text{ for }t\neq 0\text{ and }%
	g\left( 0\right) =\frac{1}{2}
	\end{equation*}%
	It is standard to verify that $g\left( t\right) $ is nondecreasing for $%
	t\geq 0$ (\cite{McDiarmid98}). Fix $\lambda >0$. We have for all $x\leq
	\lambda $ that $e^{x}\leq 1+x+g\left( \lambda \right) x^{2}$. For any $\beta 
	$ with $0<\beta \leq n\lambda $ we then have%
	\begin{eqnarray*}
		\mathbb{E}\left[ e^{\beta Y_{j}}|\sigma _{j-1}\right]  &\leq &\mathbb{E}%
		\left[ 1+\beta Y_{j}+g\left( \lambda \right) \beta ^{2}Y_{j}^{2}|\sigma
		_{j-1}\right] =1+g\left( \lambda \right) \beta ^{2}\mathbb{E}\left[
		Y_{j}^{2}|\sigma _{j-1}\right]  \\
		&\leq &\exp \left( g\left( \lambda \right) \beta ^{2}\mathbb{E}\left[
		Y_{j}^{2}|\sigma _{j-1}\right] \right) \leq \exp \left( g\left( \lambda
		\right) \left( \frac{\beta }{n}\right) ^{2}\mathbb{E}\left[ R_{j}|\sigma
		_{j-1}\right] \right) ,
	\end{eqnarray*}%
	where we also used $1+x\leq e^{x}$. Defining $Z_{0}=1$ and for $j\geq 1$ 
	\begin{equation*}
	Z_{j}=Z_{j-1}\exp \left( \beta Y_{j}-g\left( \lambda \right) \left( \frac{%
		\beta }{n}\right) ^{2}\mathbb{E}\left[ R_{j}|\sigma _{j-1}\right] \right) 
	\end{equation*}%
	then 
	\begin{equation*}
	\mathbb{E}\left[ Z_{j}|\sigma _{j-1}\right] =\exp \left( -g\left( \lambda
	\right) \left( \frac{\beta }{n}\right) ^{2}\mathbb{E}\left[ R_{j}|\sigma
	_{j-1}\right] \right) \mathbb{E}\left[ e^{\beta Y_{j}}|\sigma _{j-1}\right]
	Z_{j-1}\leq Z_{j-1}.
	\end{equation*}%
	It follows that $\mathbb{E}\left[ Z_{n}\right] \leq 1$. By induction we get%
	\begin{equation*}
	1\geq \mathbb{E}\left[ \exp \left( \beta \left( V-\hat{V}\right) -\frac{%
		g\left( \lambda \right) \beta ^{2}}{n}V\right) \right] .
	\end{equation*}%
	If we choose $\beta =n\lambda $ then%
	\begin{equation*}
	1\geq \mathbb{E}\left[ \exp \left( n\lambda \left( V-\hat{V}\right)
	-ng\left( \lambda \right) \lambda ^{2}V\right) \right] =\mathbb{E}\left[
	\exp \left( n\left( 1+2\lambda -e^{\lambda }\right) V-n\lambda \hat{V}%
	\right) \right] .
	\end{equation*}%
	Using calculus to maximize the coefficient $1+2\lambda -e^{\lambda }$ of $V$
	we set $\lambda =\ln 2$ and obtain%
	\begin{equation*}
	1\leq \mathbb{E}\left[ \exp \left( n\left( 2\ln 2-1\right) \left( V-\frac{%
		\ln 2}{2\ln 2-1}\hat{V}\right) \right) \right] .
	\end{equation*}%
	Markov's inequality then gives%
	\begin{equation*}
	\Pr \left\{ V>\frac{\ln 2}{2\ln 2-1}\hat{V}+t\right\} \leq \exp \left(
	-\left( 2\ln 2-1\right) nt\right) .
	\end{equation*}%
	To get the result we use $\ln 2/\left( 2\ln 2-1\right) \leq 2$ and $2\ln
	2-1\geq 1/e$.
\end{proof}

\subsection{Smoothness \label{Proof of Lemma gamma smooth}}

We prove (\ref{Local Gamma smooth}) and Lemma \ref{Lemma gamma smooth}. The
following is a version of Lemma \ref{Lemma gamma smooth}, where the
parameter $\gamma $ is allowed to depend on $x$.\bigskip 

\begin{lemma}
	\label{Lemma gamma-smooth}Let $f$ be a nonnegative, differentiable function
	on a convex open subset $\mathcal{O}$ of a Hilbert space and $\lambda >0$.
	Fix $x\in \mathcal{O}$ and suppose that for $\gamma >0$ and all $y\in 
	\mathcal{O}$ we have%
	\begin{equation*}
	\left\Vert f^{\prime }\left( y\right) -f^{\prime }\left( x\right)
	\right\Vert \leq \gamma \left( x\right) \left\Vert y-x\right\Vert 
	\end{equation*}%
	Then for any $y\in \mathcal{O}$%
	\begin{equation*}
	f\left( y\right) \leq \left( 1+\lambda ^{-1}\right) f\left( x\right) +\left(
	1+\lambda \right) \frac{\gamma \left( x\right) }{2}\left\Vert y-x\right\Vert
	^{2}.
	\end{equation*}%
	If $f\left( x\right) =0$ then $f\left( y\right) \leq \frac{\gamma }{2}%
	\left\Vert y-x\right\Vert ^{2}$.\bigskip 
\end{lemma}

\begin{proof}
	From the fundamental theorem of calculus we get for any $x,y\in \mathcal{X}$%
	\begin{eqnarray}
	f\left( y\right)  &=&f\left( x\right) +\int_{0}^{1}\left\langle f^{\prime
	}\left( x+t\left( y-x\right) \right) ,y-x\right\rangle dt  \notag \\
	&=&f\left( x\right) +\left\langle f^{\prime }\left( x\right)
	,y-x\right\rangle +\int_{0}^{1}\left\langle f^{\prime }\left( x+t\left(
	y-x\right) \right) -f^{\prime }\left( x\right) ,y-x\right\rangle dt  \notag
	\\
	&\leq &f\left( x\right) +\left\langle f^{\prime }\left( x\right)
	,y-x\right\rangle +\frac{\gamma \left( x\right) }{2}\left\Vert
	y-x\right\Vert ^{2}.  \label{gamma smooth bound}
	\end{eqnarray}%
	Since $f$ is nonnegative we get for any $y$ and fixed $x$%
	\begin{equation*}
	\left\langle f^{\prime }\left( x\right) ,x-y\right\rangle \leq f\left(
	x\right) +\frac{\gamma \left( x\right) }{2}\left\Vert y-x\right\Vert ^{2}.
	\end{equation*}%
	Letting $t>0$ and $x-y=tf^{\prime }\left( x\right) /\left\Vert f^{\prime
	}\left( x\right) \right\Vert $ or $y=x-tf^{\prime }\left( x\right)
	/\left\Vert f^{\prime }\left( x\right) \right\Vert $ we obtain after
	division by $t$%
	\begin{equation*}
	\left\Vert f^{\prime }\left( x\right) \right\Vert \leq \frac{f\left(
		x\right) }{t}+\frac{\gamma \left( x\right) t}{2}.
	\end{equation*}%
	Using calculus to minimize we find $\left\Vert f^{\prime }\left( x\right)
	\right\Vert \leq \sqrt{2\gamma \left( x\right) f\left( x\right) }$ and from (%
	\ref{gamma smooth bound}) and Young's inequality%
	\begin{eqnarray*}
		f\left( y\right)  &\leq &f\left( x\right) +2\sqrt{\lambda ^{-1}f\left(
			x\right) \lambda \frac{\gamma \left( x\right) }{2}\left\Vert y-x\right\Vert
			^{2}}+\frac{\gamma \left( x\right) }{2}\left\Vert y-x\right\Vert ^{2} \\
		&\leq &\left( 1+\lambda ^{-1}\right) f\left( x\right) +\left( 1+\lambda
		\right) \frac{\gamma \left( x\right) }{2}\left\Vert y-x\right\Vert ^{2}.
	\end{eqnarray*}%
	The other inequality is obtained by letting $\lambda \rightarrow 0$.\bigskip 
\end{proof}

$L\left( f,x,r\right) \leq \left\Vert f^{\prime }\left( x\right) \right\Vert
+\gamma \left( f,x\right) r/2$, and, with $\rho \left( r\right) =c\left(
1+r^{2}\right) $ for $c>0$,

To see (\ref{Local Gamma smooth}) note, that from the definition of $L\left(
f,x,r\right) $ and (\ref{gamma smooth bound}) we have%
\begin{eqnarray*}
	L\left( f,x,r\right)  &\leq &\sup_{y:0<\left\Vert y-x\right\Vert \leq r}%
	\frac{\left\langle f^{\prime }\left( x\right) ,y-x\right\rangle +\left(
		\gamma \left( f,x\right) /2\right) \left\Vert y-x\right\Vert ^{2}}{%
		\left\Vert y-x\right\Vert } \\
	&\leq &\left\Vert f^{\prime }\left( x\right) \right\Vert +\frac{\gamma
		\left( f,x\right) }{2}r.
\end{eqnarray*}%
Then with $\rho \left( r\right) =c\left( 1+r^{2}\right) $%
\begin{equation*}
\frac{L\left( f,x,r\right) }{\rho \left( r\right) }\leq \frac{\left\Vert
	f^{\prime }\left( x\right) \right\Vert +\frac{\gamma \left( f,x\right) }{2}r%
}{c\left( 1+r^{2}\right) }\leq \frac{\left\Vert f^{\prime }\left( x\right)
\right\Vert }{c}+\frac{\gamma \left( f,x\right) r}{2c\left( 1+r^{2}\right) }%
\leq \frac{\left\Vert f^{\prime }\left( x\right) \right\Vert }{c}+\frac{%
	\gamma \left( f,x\right) }{4c}
\end{equation*}%
Calculus shows that the next to last expression attains its maximum for $r>0$
at $r=1$, from which the last inequality and (\ref{Local Gamma smooth})
follow.\bigskip 

\subsection{Proof of Theorem \protect\ref{Theorem decay of G}.\label{Proof
		of Theorem decay of G}}

For the proof we need an auxiliary construction, projecting the process onto
a partion. If $\mathcal{C}=\left( C_{1},...,C_{N}\right) $ is any disjoint
partition of $\mathcal{X}$ into measurable subsets, define a process $%
\mathbf{Y}=\left( Y_{i}\right) _{i\in \mathbb{N}}$ with values in space $%
\left[ N\right] $ by $Y_{i}=j\iff X_{i}\in C_{j}$. The process $\mathbf{Y}$
inherits its ergodicity and mixing properties from $\mathbf{X}$, in the
sense that $\mathbf{Y}$ is ergodic whenever $\mathbf{X}$ is, and the mixing
coefficients of $\mathbf{Y}$ are bounded by the mixing coefficients of $%
\mathbf{X}$.

For any $x\in \mathcal{X}$ we also denote with $C\left( x\right) $ the
unique member of $\mathcal{C}$ which contains $x$.

\begin{lemma}
	\label{Lemma Decay}Assume $\left\Vert g\right\Vert _{\infty }=1$. Let $t\in
	\left( 0,1\right) $, $\tau \in \mathbb{N}$, $n\geq 2\tau /t$, and that we
	can cover the support of $\pi $ with $N$ disjoint measurable sets $%
	C_{1},...,C_{N}$ such that diam$_{g}\left( C_{j}\right) <t/2$ for all $j$.
	Then, with $m\left( n\right) =\left\lfloor nt/\left( 2\tau \right)
	\right\rfloor $,%
	\begin{equation*}
	\Pr \left\{ G\left( \mathbf{X},n,\tau ,g\right) >t\right\} \leq
	\sum_{j=1}^{N}\Pr \left\{ \exists k>nt/2,Y_{k}=j,\forall 1\leq i\leq
	nt/2-\tau ,Y_{i}\neq j\right\} .
	\end{equation*}
\end{lemma}

\begin{proof}
	Write $M\left( n\right) =m\left( n\right) \tau $. Then%
	\begin{eqnarray*}
		&&\Pr \left\{ G\left( \mathbf{X},n,\tau ,g\right) >t\right\}  \\
		&=&\Pr \left\{ \frac{1}{n-\tau }\sum_{k=\tau +1}^{n}\min_{i\in \left[ k-\tau %
			\right] }g\left( X_{k},X_{i}\right) >t\right\}  \\
		&=&\Pr \left\{ \frac{1}{n-\tau }\sum_{k:\tau <k\leq nt/2}\min_{i:1\leq i\leq
			k-\tau }g\left( X_{k},X_{i}\right) +\frac{1}{n-\tau }\sum_{k:k>nt/2}^{n}%
		\min_{i:1\leq i\leq k-\tau }g\left( X_{k},X_{i}\right) >t\right\}  \\
		&\leq &\Pr \left\{ \frac{1}{n-\tau }\sum_{k:k>nt/2}^{n}\min_{i:1\leq i\leq
			k-\tau }g\left( X_{k},X_{i}\right) >\frac{t}{2}\right\} \text{ since }\frac{%
			nt/2-\tau }{n-\tau }\leq \frac{t}{2} \\
		&\leq &\Pr \left\{ \exists k>nt/2,\min_{i:1\leq i\leq nt/2-\tau }g\left(
		X_{k},X_{i}\right) >\frac{t}{2}\right\} \text{ since }n\geq 2\tau /t\text{.}
	\end{eqnarray*}%
	Recall that diam$_{g}\left( C_{j}\right) <t/2$. Now assume that $X_{k}\in
	C_{j}$ and $\min_{i:1\leq i\leq k-\tau }g\left( X_{k},X_{i}\right) >t/2$.
	Then none of the $X_{i}$ can be in $C_{j}$ because $\sup_{y\in C_{j}}g\left(
	X_{k},y\right) \leq $ diam$_{g}\left( C_{j}\right) <t/2$. We therefore must
	have $X_{i}\notin C_{j}=C\left( X_{k}\right) $ for all $i$. Thus%
	\begin{eqnarray*}
		\Pr \left\{ G\left( \mathbf{X},n,\tau ,g\right) >t\right\}  &\leq &\Pr
		\left\{ \exists k>nt/2,\forall i:1\leq i\leq nt/2-\tau ,X_{i}\notin C\left(
		X_{k}\right) \right\}  \\
		&=&\sum_{j=1}^{N}\Pr \left\{ \exists k>nt/2,Y_{k}=j,\forall 1\leq i\leq
		nt/2-\tau ,Y_{i}\neq j\right\} .
	\end{eqnarray*}%
	\bigskip 
\end{proof}

\begin{proof}[Proof of Theorem \protect\ref{Theorem decay of G}]
	Since $\mathcal{X}$ is $g$-totally bounded there exists a partition $%
	C_{1},...,C_{N}$ such that diam$_{g}\left( C_{j}\right) <t/2$ for all $j$,
	as required for the previous lemma. We may assume that $\pi \left(
	C_{j}\right) >0$ and we work with the induced process $\mathbf{Y}$.
	
	(i) Fix $\epsilon >0$. If the underlying process is ergodic, so is the
	process $Y_{k}$, and for every $j$ we have%
	\begin{equation*}
	\Pr \left\{ \forall 1\leq i\leq n,Y_{i}\neq j\right\} \rightarrow 0\text{ as 
	}n\rightarrow \infty .
	\end{equation*}%
	Then there is $n_{0}\in \mathbb{N}$ such that forall $n\geq n_{0}$ we have $%
	\Pr \left\{ \forall 1\leq i\leq nt/2-\tau ,Y_{i}\neq j\right\} <\epsilon /N$%
	, so by the lemma 
	\begin{eqnarray*}
		\Pr \left\{ G\left( \mathbf{X},n,\tau ,g\right) >t\right\} &\leq &\Pr
		\left\{ \exists k>nt/2,\forall i:1\leq i\leq nt/2-\tau ,X_{i}\notin C\left(
		X_{k}\right) \right\} \\
		&\leq &\Pr \left\{ \exists j\in \left[ N\right] ,\forall i:1\leq i\leq
		nt/2-\tau ,Y_{i}\neq j\right\} <\epsilon .
	\end{eqnarray*}
	
	(ii) To prove almost sure convergence we use the following consequence of
	the Borel-Cantelli lemma (\cite{bauer2011probability}): let $Z_{n}$ be a
	sequence of random variables. If for every $t>0$ we have $\sum_{n\geq 1}\Pr
	\left\{ \left\vert Z_{n}\right\vert >t\right\} <\infty $ then $%
	Z_{n}\rightarrow 0$ almost surely as $n\rightarrow \infty $.
	
	Now note that the events $\left\{ \exists k>nt/2,Y_{k}=j\right\} $ and $%
	\left\{ \forall i:1\leq i\leq nt-\tau ,Y_{i}\neq j\right\} $ are separated
	by a time interval at least $\tau $, so, by the definition of the $\alpha $%
	-mixing coefficients, we have from the lemma that%
	\begin{eqnarray}
	&&\Pr \left\{ G\left( \mathbf{X},n,\tau ,g\right) >t\right\}   \notag \\
	&\leq &\sum_{j=1}^{N}\Pr \left\{ \exists k>nt/2,Y_{k}=j,\forall 1\leq i\leq
	nt/2-\tau ,Y_{i}\neq j\right\}   \notag \\
	&\leq &\sum_{j}\Pr \left\{ X\in C_{j}\right\} \Pr \bigcap_{i:1\leq i\leq 
		\frac{nt-2\tau }{2}}\left\{ \text{ }X_{i}\notin C_{j}\right\} +\alpha _{\tau
	}  \notag \\
	&\leq &\sum_{j}\Pr \left\{ X\in C_{j}\right\} \Pr \bigcap_{i:1\leq i\leq
		nt/\left( 2\tau \right) -1}\left\{ \text{ }X_{i\tau }\notin C_{j}\right\}
	+\alpha _{\tau }  \notag \\
	&\leq &\sum_{j}\pi \left( C_{j}\right) \left( 1-\pi \left( C_{j}\right)
	\right) ^{\left\lfloor nt/\left( 2\tau \right) \right\rfloor -1}+\left\lceil 
	\frac{nt}{2\tau }\right\rceil \alpha _{\tau }.  \label{estimate BK}
	\end{eqnarray}%
	Now, setting $\tau =\left\lceil n^{p}\right\rceil $ and $\alpha _{\tau }\leq
	A\tau ^{-q}$ with $p\in \left( 2/\left( 1+q\right) ,1\right) $, we obtain
	for sufficiently large $n$,%
	\begin{eqnarray*}
		\Pr \left\{ G\left( \mathbf{X},n,\left\lceil n^{p}\right\rceil ,g\right)
		>t\right\}  &\leq &N\left( 1-\min_{j}\pi \left( C_{j}\right) \right)
		^{\left\lfloor \frac{nt}{2\tau }\right\rfloor -1}+\left\lceil \frac{nt}{%
			2\tau }\right\rceil \alpha _{\tau } \\
		&\leq &N\left( 1-\min_{j}\pi \left( C_{j}\right) \right)
		^{tn^{1-p}-1}+An^{1-p-qp}t.
	\end{eqnarray*}%
	Since $1-p-qp<-1$, this expression is summable, and the conclusion follows
	from the Borel-Cantelli Lemma.
	
	(iii) Theorem 1 in \cite{berend2012missing} states that for $%
	p_{1},...,p_{N}\geq 0$, $\sum p_{i}=1$ and $m\in \mathbb{N}$ we have%
	\begin{equation*}
	\sum p_{i}\left( 1-p_{i}\right) ^{m}\leq \frac{N}{em}.
	\end{equation*}%
	Applying this with (\ref{estimate BK}) gives%
	\begin{equation*}
	\Pr \left\{ G\left( \mathbf{X},n,\tau ,g\right) >t\right\} \leq \frac{N}{%
		e\left( \left\lfloor nt/\left( 2\tau \right) \right\rfloor -1\right) }%
	+\left\lceil \frac{nt}{2\tau }\right\rceil \alpha _{\tau }.
	\end{equation*}
\end{proof}

\subsection{Stabilizing the maximal loss\label{Section softening the maximum}%
}

To allow a fixed fraction of excess errors we use a union bound over all
possible "bad" sets $B$ of fixed cardinality as in Theorem \ref{Theorem Main}%
. Now let $\alpha =\left( \left\vert B\right\vert \right) /\left( n-\tau
\right) $ be the allowed fraction of excess errors. The union bound replaces
the term $\ln \left( 1/\delta \right) $ by $\ln \left( \binom{n-\tau }{m}%
/\delta \right) $. From Stirling's approximation we can obtain for $N\in 
\mathbb{N}$ the bound%
\begin{equation*}
\ln \binom{N}{\alpha N}\leq NH\left( \alpha \right) +\text{Rest}\left(
N,\alpha \right) ,
\end{equation*}%
where $H\left( \alpha \right) =\alpha \ln \frac{1}{\alpha }+\left( 1-\alpha
\right) \ln \frac{1}{1-\alpha }$ is the entropy of an $\alpha $-Bernoulli
variable and%
\begin{equation*}
\text{Rest}\left( N,\alpha \right) =-\ln \left( \sqrt{2\pi \alpha \left(
	1-\alpha \right) N}\right) +\frac{1}{12N}\leq \left\{ 
\begin{array}{cc}
0 & \text{if }2\pi N\geq \frac{1}{\alpha \left( 1-\alpha \right) } \\ 
\frac{\ln \left( \pi N/2\right) }{2} & \text{otherwise}%
\end{array}%
\right. .
\end{equation*}%
This leads to the following version of Theorem \ref{Theorem Main} (ii),
where for simplicity we give only the second conclusion.

\begin{corollary}
	Under the condition of Theorem \ref{Theorem Main} let $\alpha \in \left[ 0,1%
	\right] $ be such that $\alpha \left( n-\tau \right) \in \mathbb{N}$. Then
	with probability at least $1-\delta $ in the sample path we have for every $%
	B\subseteq \left[ n-\tau \right] $ of cardinality $\alpha \left( n-\tau
	\right) $ and every $f\in \mathcal{F}$ that%
	\begin{multline*}
	\mathbb{E}_{\pi }\left[ f\left( X\right) \right] -\max_{i\in \left[ n-\tau %
		\right] \cap B^{c}}\Phi \left( f,X_{i}\right)  \\
	\leq \frac{2}{n-\tau }\sum_{k=\tau +1}^{n}\min_{i\in B^{c}\cap \left[ k-\tau %
		\right] }g\left( X_{k},X_{i}\right) +\left\Vert \mathcal{F}\right\Vert
	_{\infty }\varphi _{\tau }+e\left\Vert g\right\Vert _{\infty }\left( H\left(
	\alpha \right) +\frac{\text{Rest}\left( n-\tau ,\alpha \right) +\ln \left(
		1/\delta \right) }{n-\tau }\right) .
	\end{multline*}
\end{corollary}

\textbf{Remark.} Ignoring the Rest-term, which is at most of order $\ln
\left( n\right) /n$, there is an additional penalty $e\left\Vert
g\right\Vert _{\infty }H\left( \alpha \right) $ depending on the error
fraction $\alpha $. This can be interpreted as an additional empirical error
term. Since $H\left( \alpha \right) \rightarrow 0$ as $\alpha \rightarrow 0$%
, the bound tolerates a small number of excess errors. On the other hand $%
H\left( \alpha \right) /\alpha \rightarrow \infty $ logarithmically as $%
\alpha \rightarrow 0$, so the penalty is certainly exaggerated relative to a
conventional empirical error term, where the penalty would simply be $\alpha 
$.\bigskip 

\section{Experiments}

While this is primarily a theoretical paper, a small experiment has been
made to study the dependence of $G$, the dominant term in the bound, on
sample size, mixing time and recurrence time.

Let $\zeta ,\zeta _{1},\zeta _{2}\in \left[ 0,1\right] $ and let $\mathcal{Z}
$ be either the unit circle $\left[ 0,1\right] $ mod $1$ or the $2$-torus $%
\left[ 0,1\right] ^{2}$ mod $\left( 1,1\right) $ with Haar measure $\pi $.
Define the map $T:\mathcal{Z}\rightarrow \mathcal{Z}$ be the map $x\mapsto
x+\zeta $ mod $1$ (respectively $\left( x_{1},x_{2}\right) \mapsto \left(
x_{1}+\zeta _{1}\text{ mod }1,x_{2}+\zeta _{2}\text{ mod }1\right) $), and a
Markov process on $\mathcal{Z}$ by chosing a starting point in $\mathcal{Z}$
at random and the transition law 
\begin{equation*}
\Pr \left\{ A|x\right\} =\left( 1-p\right) \mathbf{1}\left\{ Tx\in A\right\}
+p\pi \left( A\right) .
\end{equation*}

Call the process so defined $\left( X_{i}\right) _{i\in \mathbb{N}}$ (or $%
\left( X_{i1},X_{i2}\right) _{i\in \mathbb{N}}$ respectively. It has
invariant distribution $\pi $ and follows iterates of the quasiperiodic
motion $Tx,T^{2}x,T^{3}x,...$ until a reset-event occurs with probability $p$
and then selects a new starting point, etc. It is exponentially mixing, with
spectral gap $p$ and relaxation time $\tau _{rel}=p^{-1}$ and mixing time $%
\tau \geq p^{-1}\ln \left( 1/\epsilon \right) $ to have distance $\epsilon $
from stationarity. The values of $\zeta ,\zeta _{1},\zeta _{2}$ are chosen
that the one-dimensional process with $p=0$ (deterministic) has a recurrence
time to intervals of length $0.01$ of ca 100 time increments, and the two
dimensional process has a recurrence time to squares of side-length $0.01$
of at least 5000.

From this we generate a sequence of images $Y_{i}\in \mathbb{R}^{784}$ by
taking one of the MNIST characters ("3") and either rotating it by the angle 
$2X_{i}\pi $ or, respectively, by rotating it by $X_{i1}$ and scaling it by
a factor $0.75+\cos \left( 2X_{i2}\pi \right) /4$ (so the scaling is in the
range $\left[ 0.5,1\right] $). The resulting image is adjusted to have mean
grey-level equal to zero and normalized to euclidean norm $1/2$ in $\mathbb{R%
}^{784}$. We generate these processes with $p=1$ (corresponding to iid
sampling), $0.1$ and $0.001$ corresponding to mixing times $\tau =1$, $19$
and $190$, the latter two chosen for TV-distance to independence of  $%
\epsilon =0.1$. We computed the value of $G$ for sample sizes $2^{n}$ with $%
n=1,...,12$, whenever the sample size exceeded the mixing time. The results
are shown in Figure \ref{Figure1} and Figure \ref{Figure2} respectively. They show, the
dimension dependent decay of $G$. For the slower mixing rate we also plotted 
$\ln \left( \tau /n\right) $ at those values, where this exceeds $\ln \left(
G\right) $, indicating that the bound in Theorem \ref{Theorem Main} is
better than bounds decaying as $\tau /n$ (this comparison ignores constants
and possible complexity values in the other bound). This indicates, that
Theorem \ref{Theorem Main} is preferable whenever the mixing time exceeds
the recurrence time.

\begin{figure}[h!]
	\includegraphics[width=\linewidth]{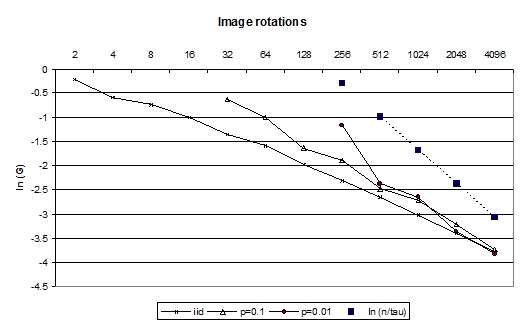}  
	\caption{The estimator $G$ on rotated images for different sample sizes and reset probabilities. We also plot an optimistic estimate on competing bounds}
	\label{Figure1}
\end{figure}

\begin{figure}[h!]
	\includegraphics[width=\linewidth]{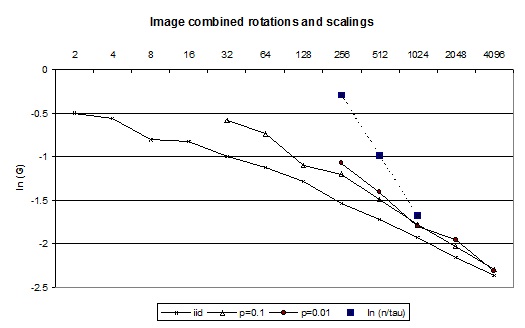}  
	\caption{The estimator $G$ on rotated and scaled images for different sample sizes and reset probabilities. We also plot an optimistic estimate on competing bounds}
	\label{Figure2}
\end{figure}

\end{document}